\documentclass[12pt]{colt2019} 

\usepackage{morenotations}
\title[Adversarial Networks and Autoencoders]{\papertitle}
\usepackage{times}

\coltauthor{%
 \Name{Hisham Husain} \Email{hisham.husain@anu.edu.au}\\
 \addr Australian National University \& Data61
 \AND
 \Name{Richard Nock} \Email{richard.nock@data61.csiro.au}\\
 \addr Data61, the Australian National University, \& the University of Sydney%
 \AND
 \Name{Robert C. Williamson} \Email{bob.williamson@anu.edu.au}\\
 \addr Australian National University \& Data61%
}

\begin{document}
\clearpage\maketitle
\thispagestyle{empty}

\begin{abstract}%
  Since the introduction of Generative Adversarial Networks (GANs) and Variational Autoencoders (VAE), the literature on generative modelling has witnessed an overwhelming resurgence. The impressive, yet elusive empirical performance of GANs has lead to the rise of many GAN-VAE hybrids, with the hopes of GAN level performance and additional benefits of VAE, such as an encoder for feature reduction, which is not offered by GANs. Recently, the Wasserstein Autoencoder (WAE) was proposed, achieving performance similar to that of GANs, yet it is still unclear whether the two are fundamentally different or can be further improved into a unified model. In this work, we study the $f$-GAN and WAE models and make two main discoveries. First, we find that the $f$-GAN and WAE objectives partake in a primal-dual relationship and are equivalent under some assumptions, which then allows us to explicate the success of WAE. Second, the equivalence result allows us to, for the first time, prove generalization bounds for Autoencoder models, which is a pertinent problem when it comes to theoretical analyses of generative models. Furthermore, we show that the WAE objective is related to other statistical quantities such as the $f$-divergence and in particular, upper bounded by the Wasserstein distance, which then allows us to tap into existing efficient (regularized) optimal transport solvers. Our findings thus present the first primal-dual relationship between GANs and Autoencoder models, comment on generalization abilities and make a step towards unifying these models.

\end{abstract}


\section{Introduction}
Implicit probabilistic models \citep{mohamed2016learning} are defined to be the pushforward of a simple distribution $P_Z$ over a latent space $\mathcal{Z}$ through a map $G: \mathcal{Z} \to \mathcal{X}$, where $\mathcal{X}$ is the space of the input data. Such models allow easy sampling, but the computation of the corresponding probability density function is intractable. The goal of these methods is to match $G \# P_Z$ to a target distribution by minimizing $D(P_X,  G\# P_Z)$, for some discrepancy $D(\cdot,\cdot)$ between distributions. An overwhelming number of methods have emerged after the introduction of Generative Adversarial Networks \citep{goodfellow2014generative, nowozin2016f} and Variational Autoencoders \citep{kingma2013auto} (GANs and VAEs), which have established two distinct paradigms: Adversarial (networks) training and Autoencoders respectively. Adversarial  involves a set of functions $\mathcal{D}$, referred to as \textit{discriminators}, with an objective of the form
\begin{align}
  \label{first:adversarial}
D(P_X, G \# P_Z) = \max_{d \in \mathcal{D}} \braces{\E_{x \sim P_X}[a(d(x))] - \E_{x \sim G \# P_Z}[b(d(x))]},
\end{align}
for some functions $a: \mathbb{R} \to \mathbb{R}$ and $b: \mathbb{R} \to \mathbb{R}$. Autoencoder methods involve finding a function $E: \mathcal{X} \to \mathcal{Z}$, referred to as an \textit{encoder}, whose goal is to reverse $G$, and learn a feature space with the objective
\begin{align}
  \label{first:encoder}
D(P_X, G \# P_Z) = \min_{E}  \braces{\mathcal{R}(G,E) + \Omega(E)},
\end{align}
where $\mathcal{R}(G,E)$ is the \textit{reconstruction} loss and acts to ensure $G$ and $E$ reverse each other and $\Omega(E)$ is a regularization term. Much work on Autoencoder methods has focused upon the choice of $\Omega$.

Both methods have their own strengths and limitations, along with differing directions of progress. Indeed, there is a lack of theoretical understanding of how these frameworks are parametrized and it is not clear whether the methods are fundamentally different. For example, Adversarial training based methods have empirically demonstrated high performance when it comes to producing realistic looking samples from $P_X$. However, GANs often have problems in convergence and stability of training \citep{goodfellow2016nips}. Autoencoders, on the other hand, deal with a more well behaved objective and learn an encoder in the process, making them useful for feature representation. However in practice, Autoencoder based methods have reported shortfalls in practice, such as producing blurry samples for image based datasets \citep{tolstikhin2017wasserstein}. This has motivated researchers to borrow elements from Adversarial training in the hopes of achieving GAN performance. Examples include replacing $\Omega$ with Adversarial objectives \citep{mescheder2017adversarial, makhzani2015adversarial} or replacing the reconstruction loss with an adversarial objective \citep{dumoulin2016adversarially, alanov2018pairwise}. Recently, the Wasserstein Autoencoder (WAE) \citep{tolstikhin2017wasserstein} has been shown to subsume these two methods, with an Adversarial based $\Omega$ and has demonstrated performance similar to that of Adversarial methods.

When it comes to directions of progress, Adversarial training methods now have theoretical guarantees on generalization performance \citep{zhang2017discrimination}, however no such theoretical results have been obtained to date for autoencoders. Generalization performance is a pressing concern, since both techniques implicitly assume the samples represent the target distribution \citep{li2018implicit}. A formal connection will benefit both methods, allowing them to inherit strengths from one another.

In this work, we study the two paradigms and in particular focus on the $f$-GANs \citep{nowozin2016f} for Adversarial training and Wasserstein Autoencoders (WAE) for Autoencoders, which generalize the original GAN and VAE models respectively. We prove that the $f$-GAN objective with Lipschitz (with respect to a metric $c$) discriminators is equivalent to the WAE objective with cost $c$. In particular, we show that the WAE objective is an upper bound so as to have
\begin{align*}
f\textrm{-GAN} \leq \textrm{WAE},
\end{align*}
and discuss the tightness of this bound. Our result is a generalization of the Kantorovich-Rubinstein duality and thus suggests a primal-dual relationship between Adversarial and Autoencoder methods. Consequently we show, to the best of our knowledge, the first generalization bounds for autoencoders. Furthermore, using this equivalence, we show that the WAE objective is related to key statistical quantities such as the $f$-divergence and Wasserstein distance, which allows us to tap into efficient (regularized) OT solvers.

We also present another contribution regarding the parametrization of WAE in the Appendix, relating optimization of Brenier potentials in transport theory, to the WAE objective (Section \ref{appendix-contribution}). The main contributions can be summarized as the following: \\[0.5\baselineskip]
\noindent $\triangleright$ (Theorem \ref{fgan-wae-equivalence}) Establish an equivalence between Adversarial training and Wasserstein Autoencoders, showing conditions under which the $f$-GAN and WAE coincide. This further justifies the similar performance of WAE to GAN based methods. When the conditions are not met, we have an inequality, which allows us to comment on the behavior of the methods. \\[0.5\baselineskip]
\noindent $\triangleright$ (Theorem \ref{fdiv-fwae}, \ref{final-equality-condition} and \ref{fwae-entropy}) Show that the WAE objective is related to other statistical quantities such as $f$-divergence, Wasserstein distance and the entropy regularized Wasserstein distance. \\[0.5\baselineskip]
\noindent $\triangleright$ (Theorem \ref{main-gen-result}) Provide generalization bounds for WAE. In particular, this focuses on the empirical variant of the WAE objective, which allows the use of OT solvers as they are concerned with discrete distributions. This allows one to employ efficient (regularized) OT solvers for the estimation of WAE, $f$-GANs and the generalization bounds.

\section{Preliminaries}
\subsection{Notation}
We will use $\mathcal{X}$ to denote the input space (a Polish space), typically taken to be a Euclidean space. We use $\mathcal{Z}$ to denote the latent space, also taken to be Euclidean. We use $\mathbb{N}_{*}$ to denote the natural numbers without $0$: $\mathbb{N} \setminus \braces{0}$. The set $\mathscr{P}$ contains the set of probability measures over $\mathcal{X}$, and elements of this set will be referred to \textit{distributions}. If $P \in \mathscr{P}(\mathcal{X})$ happens to be absolutely continuous with respect to the Lebesgue measure then we will use $dP/dx$ to refer to the \textit{density} function (Radon-Nikodym derivative with respect to the Lebesgue measure). For any $T \in \mathscr{F}(\mathcal{X}, \mathcal{Z})$, for any measure $\mu \in \mathscr{P}(\mathcal{X})$, the pushforward measure of $\mu$ through $T$ denoted $T \# \mu \in \mathscr{P}(\mathcal{Z})$ is such that $T \# \mu (A) = \mu(T^{-1}(A))$ for any measurable set $A \subset \mathcal{Z}$. The set $\mathscr{F}(\mathcal{X}, \mathbb{R})$ refers to all measurable functions from $\mathcal{X}$ into the set $\mathbb{R}$. We will use functions to represent conditional distributions over a space $\mathcal{Z}$ conditioned on elements $\mathcal{X}$, for example $P \in \mathscr{F}(\mathcal{X}, \mathscr{P}(\mathcal{Z}))$ so that for any $x \in \mathcal{X}$, $P(x) = P(\cdot | x) \in \mathscr{P}(\mathcal{Z})$. For any $P \in \mathscr{P}(\mathcal{X})$, the \textit{support} of $P$ is $\text{supp}(P) = \braces{x \in \mathcal{X} : \text{ if }x \in N_x\text{ open} \implies P(N_x) > 0 }$. In any metric space $(\mathcal{X}, c)$, for any set $S \subseteq \mathcal{X}$, we define the \textit{diameter} of $S$ to be $\text{diam}_c(S) = \sup_{x,x' \in S} c(x,x')$. For a metric $c$ over $\mathcal{X}$, then for any $f \in \mathscr{F}(\mathcal{X}, \mathbb{R})$, $\text{Lip}_{c}(f)$ denotes the Lipschitz constant of $f$ with respect to $c$ and $\mathcal{H}_c = \braces{g \in \mathscr{F}(\mathcal{X}, \mathbb{R}) : \text{Lip}_{c}(g) \leq 1}$. For some set $S \subseteq \mathbb{R}$, $\mathbf{1}_S$ corresponds to the convex \textit{indicator function}, ie. $\mathbf{1}_S(x) = 0$ if $x \in S$ and $\mathbf{1}_S(x) = \infty$ otherwise. For any $ x\in \mathcal{X}$, $\delta_{x}: \mathcal{X} \to \braces{0,1}$ corresponds to the \textit{characteristic function}, with $\delta_x(0) = 1$ if $x = 0$ and $\delta_x(0) = 0$ if $x \neq 0$.
\subsection{Background}
\label{sec:background}
\subsubsection{Probability Discrepancies}
Probability discrepancies are central to the objective of finding the best fitting model. We introduce some key discrepancies and their notation, which will appear later.
\begin{definition}[$f$-Divergence]
For a convex function $f:\mathbb{R} \to (-\infty,\infty]$ with $f(1) = 0$, for any $P,Q \in \mathscr{P}(\mathcal{X})$ with $P$ absolutely continuous with respect to $Q$, the $f$-Divergence between $P$ and $Q$ is
\begin{align*}
D_f(P,Q) := \int_{\mathcal{X}} f\bracket{\frac{dP}{dQ}} dQ.
\end{align*}
\end{definition}
In order to compute the $f$-divergence, one can first compute $dP/dQ$ and estimate the integral empirically using samples from $Q$.
\begin{definition}[Integral Probability Metric]
For a fixed function class $\mathcal{F} \subseteq \mathscr{F}(\mathcal{X}, \mathbb{R})$, the Integral Probability Metric (IPM) based on $\mathcal{F}$ between $P,Q \in \mathscr{P}(\mathcal{X})$ is defined as
\begin{align*}
\operatorname{IPM}_{\mathcal{F}}(P,Q) := \sup_{f \in \mathcal{F}} \braces{\int_{\mathcal{X}}f(x) dP(x) - \int_{\mathcal{X}} f(x) dQ(x) }.
\end{align*}
\end{definition}
If we have that $-\mathcal{F} = \mathcal{F}$ then $\operatorname{IPM}_{\mathcal{F}}$ forms a metric over $\mathscr{P}(\mathcal{X})$ \citep{muller1997integral}. A particular IPM we will make use of is Total Variation (TV): $\text{TV}(P,Q) = \operatorname{IPM}_{\mathcal{V}}(P,Q)$ where $\mathcal{V} = \braces{h \in \mathscr{F}(\mathcal{X}, \mathbb{R}) : \card{h} \leq 1}$. We also note that when $f(x) = \card{x-1}$ then $\text{TV} = D_f$ and thus TV is both an IPM and an $f$-divergence.
\begin{definition}
For any $P,Q \in \mathscr{P}(\mathcal{X})$, define the \emph{set of couplings} between $P$ and $Q$ to be
\begin{align*}
\Pi(P,Q) = \braces{\pi \in \mathscr{P}(\mathcal{X} \times \mathcal{X}) : \int_{\mathcal{X}} \pi(x,y) dx = P, \int_{\mathcal{X}} \pi(x,y) dy = Q}.
\end{align*}
For a cost $c: \mathcal{X} \times \mathcal{X} \to \mathbb{R}_{+}$, the \emph{Wasserstein distance} between $P$ and $Q$ is
\begin{align*}
W_c(P,Q) := \inf_{\pi \in \Pi(P,Q)} \braces{\int_{\mathcal{X} \times \mathcal{X}} c(x,y) \pi(x,y) }.
\end{align*}
\end{definition}
The Wasserstein distance can be regarded as an infinite linear program and thus admits a dual form, and in the case of $c$ being a metric, belongs to the class of IPMs, which we summarize in the following lemma \citep{vOT}.
\begin{lemma}[Wasserstein Duality]
	\label{w-duality}
Let $(\mathcal{X}, c)$ be a metric space, and suppose $\mathcal{H}_c$ is the set of all $1$-Lipschitz functions with respect to $c$. Then for any $P,Q \in \mathscr{P}(\mathcal{X})$, we have
\begin{align*}
W_c(P,Q) &= \sup_{h \in \mathcal{H}_c} \braces{\int_{\mathcal{X}} h(x) dP(x) - \int_{\mathcal{X}} h(x) dQ(x) }\\
	      &= \operatorname{IPM}_{\mathcal{H}_c}(P,Q)
\end{align*}
\end{lemma}
\subsection{Generative Models}
\label{background:gen-models}
In both GAN and VAE models, we have a latent space $\mathcal{Z}$ (typically taken to be $\mathbb{R}^d$, with $d$ being small) and a prior distribution $P_Z \in \mathscr{P}(\mathcal{Z})$ (eg. unit variance Gaussian). We have a function referred to as the generator $G: \mathcal{Z} \to \mathcal{X}$, which induces the \textit{generated} distribution, denoted by $P_G \in \mathscr{P}(\mathcal{X})$, as the pushforward of $P_Z$ through $G$: $P_G = G \# P_Z$. The true data distribution will be referred to as $P_X \in \mathscr{P}(\mathcal{X})$. The common goal between the two methods is to find a generator $G$ such that the samples generated by pushing forward $P_Z$ through $G$ ($G \# P_Z$) are close to the true data distribution ($P_X$). More formally, one can cast this as an optimization problem by finding the best $G$ such that $D(P_G, P_X)$ is minimized where $D(\cdot, \cdot)$ is some discrepancy between distributions. Both methods (as we outline below) utilize their own discrepancies between $P_X$ and $P_G$, which offer their own benefits and weaknesses.
\subsubsection{Wasserstein Autoencoder}
Let $E: \mathcal{X} \to \mathscr{P}(\mathcal{Z})$ denote a probabilistic \textit{encoder}, which maps each point $x$ to a conditional distribution $E(x) \in \mathscr{P}(\mathcal{Z})$, denoted as the \textit{posterior} distribution. The pushforward of $P_X$ through $E$: $E \# P_X$, will be referred to as the \textit{aggregated posterior}.
\begin{definition}[Wasserstein Autoencoder \citep{tolstikhin2017wasserstein}]
Let $c: \mathcal{X} \times \mathcal{X} \to \mathbb{R}_{\geq 0}$, $\lambda > 0$ and $\Omega: \mathscr{P}(\mathcal{Z}) \times \mathscr{P}(\mathcal{Z}) \to \mathbb{R}_{\geq 0}$ with $\Omega(P,P) = 0$ for all $P \in \mathscr{P}(\mathcal{Z})$. The Wasserstein Autoencoder objective is
\begin{align*}
\operatorname{WAE}_{c, \lambda \cdot \Omega}(P_X,G) = \inf_{E \in \mathscr{F}(\mathcal{X}, \mathscr{P}(\mathcal{Z}))} \braces{\int_{\mathcal{X}} \E_{z \sim E(x)}[c(x,G(z))] dP_X(x) +  \lambda \cdot \Omega (E \# P_X,  P_Z)}
\end{align*}
\end{definition}
We remark that there are various choices of $c$ and $\lambda \cdot \Omega$. \citet{tolstikhin2017wasserstein} select these by tuning $\lambda$ and selecting different probability distortions for $\Omega$.
\subsubsection{$f$-Generative Adversarial Network}
Let $d: \mathcal{X} \to \mathbb{R}$ denote a \textit{discriminator} function.
\begin{definition}[$f$-GAN \citep{nowozin2016f}]
Let $f:\mathbb{R} \to (-\infty,\infty]$ denote a convex function with property $f(1) = 0$ and $\mathcal{D} \subset \mathscr{F}(\mathcal{X}, \mathbb{R})$ a set of discriminators. The $f$-GAN model minimizes the following objective for a generator $G: \mathcal{Z} \to \mathcal{X}$
\begin{align}
\label{gan-objective}
\operatorname{GAN}_f(P_X,G; \mathcal{D}) := \sup_{d \in \mathcal{D}}\braces{\E_{x \sim P_X}[d(x)] - \E_{z \sim P_Z}[f^{*}(d(G(z)))]},
\end{align}
where $f^{\star}(x) = \sup_{y}\braces{x \cdot y - f(y)}$ is the convex conjugate of $f$.
\end{definition}
There are two knobs in this method, namely $\mathcal{D}$, the set of discriminators and the convex function $f$. The objective in (\ref{gan-objective}) is a variational approximation to $D_f$ \citep{nowozin2016f}; if $\mathcal{D} = \mathscr{F}(\mathcal{X}, \mathbb{R})$, then $\operatorname{GAN}_f(P_X,G;\mathcal{D}) = D_f(P_X,P_G)$ \citep{nguyen2010estimating}. In the case of $f(x) = x \log(x) - (x+1)\log(x+1) + 2 \log 2$, we recover the original GAN \citep{goodfellow2014generative}.

\section{Related Work}
Current attempts at building a taxonomy for generative models have largely been within each paradigm or the proposal of hybrid methods that borrow elements from the two. We first review major and relevant advances in each paradigm, and then move on to discuss results that are close to the technical contributions of our work.

The line of Autoencoders begin with $\Omega = 0$, which is the original autoencoder concerned only with reconstruction loss. VAE then introduced a non-zero $\Omega$, along with implementing Gaussian encoders \citep{kingma2013auto}. This was then replaced by an adversarial objective \citep{mescheder2017adversarial}, which is sample based and consequently allows arbitrary encoders. In the spirit of unification, Adversarial Autoencoders (AAE) \citep{makhzani2015adversarial} proposed $\Omega$ to be a discrepancy between the pushforward of the target distribution through the encoder ($E\# P_X$) and the prior distribution ($P_Z$) in the latent space, which was then figured out to be equivalent to the VAE $\Omega$ minus a mutual information term \citep{hoffman2016elbo}. Independently, InfoVAE \citep{zhao2017infovae} proposed a similar objective,  which was then figured out to be equivalent to adding mutual information. \citet{tolstikhin2017wasserstein} then reparametrized the Wasserstein distance into an Autoencoder objective (WAE) where the $\Omega$ term generalizes AAE, and has reported performance comparable to that of Adversarial methods. Other attempts also include adjusting the reconstruction loss to be adversarial as well \citep{dumoulin2016adversarially,alanov2018pairwise}. Another work that focuses on WAE is the Sinkhorn Autoencoders (SAE) \citep{patrini2018sinkhorn}, which select $\Omega$ to be the Wasserstein distance and show that the overall objective is an upper bound to the Wasserstein distance between $P_X$ and $P_G$.

\citet{hu2017unifying} discussed the two paradigms and their unification by interpretting GANs from the perspective of variational inference, which allowed a connection to VAE, resulting in a GAN implemented with importance weighting techniques. While this approach is the closest to our work in forming a link, their results apply to standard VAE (and not other AE methods such as WAE) and cannot be extended to all $f$-GANs. \citet{liu2017approximation} introduced the notion of an Adversarial divergence, which subsumed mainstream adversarial based methods. This also lead to the formal understanding of how the selected discriminator set $\mathcal{D}$ affects the final $G$ learned. However, this approach is silent with regard to Autoencoder based methods. \citet{zhang2017discrimination} established the tradeoff between the Rademacher complexity of the discriminator class $\mathcal{D}$ and generalization performance of $G$, with no results present for Autoencoders. These theoretical advances in Adversarial training methods are inherited by Autoencoders as a consequence of the equivalence presented in our work.

One key point in the proof of our equivalence is the use of a result that decomposes the GAN objective into an $f$-divergence and an IPM for a restricted class of discriminators (which we used for Lipschitz functions). This decomposition is used in \citep{liu2018inductive} and applied to linear $f$-GANs, showing that the adversarial training objective decomposes into a mixture of maximum likelihood and moment matching. \citet{NIPS2018_7771} used this decomposition with Lipschitz discriminators like our work, however does not make any extension or further progress to establish the link to WAE. Indeed, GANs with Lipschitz discriminators have been independently studied in \citep{zhou2018understanding}, which suggest that one should enforce Lipschitz constraints to provide useful gradients.

\section{$f$-Wasserstein Autoencoders}
In the sequel, for any $G \in \mathscr{F}(\mathcal{Z}, \mathcal{X})$ considered, we will be assuming that $\mathcal{X} \subseteq \text{Im}(G)$. We introduce an objective, which we refer to as the $f$-Wasserstein Autoencoder, that will help us in the proof of the main theorems of this paper.
\begin{definition}[$f$-Wasserstein Autoencoder]
Let $c: \mathcal{X} \times \mathcal{X} \to \mathbb{R}$, $\lambda > 0$, $f:\mathbb{R} \to (-\infty,\infty]$ be a convex function (with $f(1) = 0$) and $P_X,P_G$ defined in Section \ref{background:gen-models}. We define the $f$-Wasserstein Autoencoder ($f$-WAE) objective to be
\begin{align}
\overline{W}_{c,\lambda \cdot f}(P_X,G) = \inf_{E \in \mathscr{F}(\mathcal{X}, \mathscr{P}(\mathcal{Z}))} \braces{W_c(P_X, (G \circ E) \# P_X) +  \lambda D_f (E \# P_X,  P_Z)} \label{fwae-def}
\end{align}
\end{definition}
In the proof of the main result, we will show that the $f$-WAE objective is indeed the same as the WAE objective when using the same cost $c$ and selecting the regularizer to be $\lambda \cdot \Omega = D_{\lambda f} = \lambda D_f$. The only difference between this and the standard WAE is the use of $W_c( (G \circ E) \# P_X,P_X)$ as reconstruction instead of the standard cost which is an upper bound (Lemma \ref{fgan-wae:three}), and the regularizer is chosen to be $\lambda \cdot \Omega = D_{\lambda f} = \lambda D_f$. We now present the main theorem that captures the relationship between $f$-GANs, $f$-WAE and WAE.
\begin{theorem}[$f$-GAN and WAE equivalence]
\label{fgan-wae-equivalence}
Suppose $c$ is a metric and let $\mathcal{H}_c$ denote the set of all functions from $\mathcal{X} \to \mathbb{R}$ that are $1$-Lipschitz (with respect to $c$). Let $f:\mathbb{R} \to (-\infty,\infty]$ be a convex function with $f(1) = 0$, then we have for all $\lambda > 0$,
\begin{align}
\operatorname{GAN}_{\lambda f}(P_X,G; \mathcal{H}_c) \leq \operatorname{WAE}_{c,\lambda \cdot D_f}(P_X,G), \label{equivalence-1}
\end{align}
with equality if $G$ is invertible.
\end{theorem}
\begin{proof}
\textbf{(This is a sketch, see Section \ref{proof:fgan-wae-equivalence} for full proof)}. The proof begins by proving certain properties of $\mathcal{H}_c$ (Lemma \ref{hc-regular}), allowing us to use the dual form of restricted GANs (Theorem \ref{liu-theorem}),
\begin{align}
\operatorname{GAN}_{f}(P_X,G; \mathcal{H}_c) &= \inf_{P' \in \mathscr{P}(\mathcal{X})} \braces{D_f(P',P_G) + \sup_{h \in \mathcal{H}_c}\braces{\E_{P_X}[h] - \E_{P'}[h] } } \nonumber\\
                                             &= \inf_{P' \in \mathscr{P}(\mathcal{X})} \braces{D_f(P',P_G) + W_c(P',P_X)}. \label{dual-fgan}
\end{align}
The key is to reparametrize (\ref{dual-fgan}) by optimizing over couplings. By rewriting $P' = (G \circ E) \# P_X$ for some $E \in \mathscr{F}(\mathcal{X}, \mathscr{P}(\mathcal{Z}))$ and rewriting (\ref{dual-fgan}) as an optimization over $E$. This is justified by Lemma \ref{reparamP-lem}. We obtain
\begin{align}
&\inf_{P' \in \mathscr{P}(\mathcal{X})} \braces{D_f(P',P_G) + W_c(P',P_X)} \nonumber \\
&= \inf_{E \in \mathscr{F}(\mathcal{X}, \mathscr{P}(\mathcal{Z}))} \braces{D_f((G \circ E) \# P_X,P_G) + W_c((G \circ E) \# P_X,P_X)} \label{equivalence:one}
\end{align}
We then have
\begin{align*}
D_f((G \circ E) \# P_X,P_G) = D_f(G \# (E \# P_X), G \# P_Z)
                           &\stackrel{(*)}{\leq} D_f(E\#P_X, P_Z),
\end{align*}
with equality in $(*)$ if $G$ is invertible (Lemma \ref{fgan-wae:two}). A weaker condition is required if $f$ is differentiable, namely if $G$ is invertible with respect to $f' \circ d(E\#P_X)/dP_Z$ in the sense that
\begin{align}
G(z) = G(z') \implies  f' \circ (d(E\#P_X)/dP_Z)(z) =  f' \circ (d(E\#P_X)/dP_Z)(z'), \label{condition}
\end{align}
noting that an invertible $G$ trivially satisfies this requirement. Letting $f \gets \lambda f$, we have $D_{f} (\cdot,\cdot) \gets \lambda D_{f} (\cdot,\cdot)$, and so from Equation \ref{equivalence:one}, we have
\begin{align*}
\operatorname{GAN}_{\lambda f}(P_X,G; \mathcal{H}_c) &\stackrel{(*)}{\leq} \inf_{E \in \mathscr{F}(\mathcal{X}, \mathscr{P}(\mathcal{Z}))} \braces{\lambda D_f(E\#P_X, P_Z) + W_c((G \circ E) \# P_X,P_X)}\\
                                             &= \overline{W}_{c,\lambda \cdot f}\\
                                             &\leq \inf_{E \in \mathscr{F}(\mathcal{X}, \mathscr{P}(\mathcal{Z}))} \braces{  \lambda D_f (E \# P_X,  P_Z) + \int_{\mathcal{X}} \E_{z \sim E(x)}[c(x,G(z))] dP_X(x)}\\
                                             &= \operatorname{WAE}_{c,\lambda \cdot D_f}(P_X,G),
\end{align*}
where the final inequality follows from the fact that $W_c(P,Q)  \leq  \int_{\mathcal{X}} \E_{z \sim E(x)}[c(x,G(z))] dP_X(x)$ (Lemma \ref{fgan-wae:three}). Using the fact that $\overline{W} \geq \operatorname{WAE}$ (Lemma \ref{f-WAEisWAE}) completes the proof.
\end{proof}
When $G$ is invertible, we remark that $P_G$ can still be expressive and capable of modelling complex distributions in WAE and GAN models. For example, if $G$ is implemented with feedforward neural networks, and $G$ is invertible then $P_G$ can model \textit{deformed} exponential families \citep{nock2017f}, which encompasses a large class appearing in statistical physics and information geometry \citep{amari2016information, borland1998ito}. There exists many invertible activation functions under which $G$ will be invertible. Furthermore, in the proof of the Theorem it is clear that $\overline{W}$ and $\operatorname{WAE}$ are the same objective (from Lemma \ref{fgan-wae:three} and Lemma \ref{f-WAEisWAE}). When using $f = \mathbf{1}_{\braces{1}}$ ($f(x) = 0$ if $x = 1$ and $f(x) = \infty$ otherwise), and noting that $f^{\star}(x) = x$, meaning that Theorem \ref{fgan-wae-equivalence} (with $\lambda = 1$) reduces to
\begin{align*}
\sup_{h \in \mathcal{H}_c} \braces{\E_{x \sim P_X}[h(x)] - \E_{x \sim P_G}[h(x)]} &= \operatorname{GAN}_{f}(P_X,G; \mathcal{H}_c)\\
                              &\leq \overline{W}_{c,f}(P_X,P_G)\\
                              &= \inf_{E \in \mathscr{F}(\mathcal{X}, \mathscr{P}(\mathcal{Z})) : E \# P_X = P_Z } \braces{W_c(P_X, (G \circ E) \# P_X)}\\
                              &= \inf_{E \in \mathscr{F}(\mathcal{X}, \mathscr{P}(\mathcal{Z})) : E \# P_X = P_Z } \braces{W_c(P_X, G \# P_Z}\\
                              &= W_c(P_X,P_G),
\end{align*}
which is the standard primal-dual relation between Wasserstein distances as in Lemma \ref{w-duality}. Hence, Theorem \ref{fgan-wae-equivalence} can be viewed as a generalization of this primal-dual relationship, where Autoencoder and Adversarial objectives represent primal and dual forms respectively.

We note that the left handside of Equation (\ref{equivalence-1}) does not explicitly engage the prior space $Z$ as much as the right hand side in the sense that one can set $Z = \mathcal{X}$, $G = \text{Id}$ (which is invertible) and $P_Z = P_G$ and indeed results in the exact same $f$-GAN objective since $G \# P_Z = \text{Id} \# P_G = P_G$, yet the equivalent $f$-WAE objective (from Theorem \ref{fgan-wae-equivalence}) will be different. This makes the Theorem versatile in reparametrizations, which we exploit in the proof for Theorem \ref{final-equality-condition}. We now consider weighting the reconstruction along with the regularization term in $\overline{W}$ (which is equivalent to weighting WAE), which simply amounts to re-weighting the cost since for any $\gamma > 0$,
  \begin{align*}
  \overline{W}_{\gamma \cdot c, \lambda \cdot f}(P_X,G) = \inf_{E \in \mathscr{F}(\mathcal{X}, \mathscr{P}(\mathcal{Z}))} \braces{\gamma W_c((G \circ E) \# P_X, P_X) + \lambda D_f(E \# P_X,P_Z) }.
\end{align*}
The idea of weighting the regularization term by $\lambda$ was introduced by \citep{higgins2016beta} and furthermore studied empirically, showing that the choice of $\lambda$ influences learning disentanglement  in the latent space. \citep{alemi2018fixing}. We show that if $\lambda = 1$ and  $\gamma$ is larger than some $\gamma^{*}$ then $\overline{W}$ will become an $f$-divergence (Theorem \ref{fdiv-fwae}). On the other hand if we fix $\gamma = 1$ and take $\lambda$ is larger than some $\lambda^{*}$, then $\overline{W}$ becomes the Wasserstein distance and in particular, all equalities hold in (\ref{equivalence-1}) (Theorem \ref{final-equality-condition}). We show explicitly how high $\gamma$ and $\lambda$ need to be for such equalities to occur. Since $f$-divergence and Wasserstein distance are quite different distortions in terms of their properies, this gives an interpretation on weighting each term.

We now outline the $f$-divergence case. We will be focusing on $f:\mathbb{R} \to (-\infty,\infty]$ convex, differentiable and $f(1) = 0$. In the case we assume that $P_X$ is absolutely continuous with respect to $P_G$, so that $D_f(P_X,P_G) < \infty$. We then have the following
\begin{theorem}
\label{fdiv-fwae}
Set $c(x,y) = \delta_{x-y}$ and let $f:\mathbb{R} \to (-\infty,\infty]$ be a convex function (with $f(1) = 0$) and differentiable. Let $\gamma^{*} = \sup_{x \in \mathcal{X}} \card{f'\bracket{\frac{dP_X}{dP_G}} - f'(0) }$ and suppose $P_G$ is absolutely continuous with respect to $P_X$ and that $G$ is invertible, then we have for all $\gamma \geq \gamma^{*}$
\begin{align*}
\overline{W}_{\gamma \cdot c, f}(P_X,G) = D_f(P_X,P_G).
\end{align*}
\end{theorem}
(Proof in Appendix, Section \ref{proof:fdiv-fwae}). One may actually pick the following value
\begin{align*}
\sup_{x,x' \in \mathcal{X}} \card{f'\bracket{\frac{dP_X}{dP_G}}(x) - f'\bracket{\frac{dP_X}{dP_G}}(x') } \leq \gamma^{*},
\end{align*}
for $\gamma$ for Theorem \ref{fdiv-fwae} to hold, noting that it is smaller than $\gamma^{*}$ since $f'$ is increasing ($f$ is convex) and  $dP_X/dP_G > 0$. It is important to note that $W_c (P_X,P_G) = \text{TV}(P_X,P_G)$ when $c(x,y) = \delta_{x-y}$ and so Theorem \ref{fdiv-fwae} tells us that the objective with a weighted total variation reconstruction loss with a $f$-divergence prior regularization amounts to the $f$-divergence. It was shown that in \citep{nock2017f} that when $G$ is an invertible feedforward neural network then $D_f(P_X,P_G)$ is a \textit{bregman} divergence (a well regarded quantity in information geometry) between the parametrizations of the network for a particular choice of activation function for $G$, which depends on $f$. Hence, a practioner should design $G$ with such activation function when using $f$-WAE under the above setting ($c(x,y) = \delta_{x-y}$ and $\gamma = \gamma^{*}$) with $G$ being invertible, so that the information theoretic divergence ($D_f$) between the distributions becomes an information geometric divergence involving the network parameters.

We now show that if $\lambda$ is selected high enough then $\overline{W}$ becomes $W_c$ and furthermore we have equality between $f$-GAN, $f$-WAE and WAE.
\begin{theorem}
  \label{final-equality-condition}
Let $c: \mathcal{X} \times \mathcal{X} \to \mathbb{R}$ be a metric. For any $f:\mathbb{R} \to (-\infty,\infty]$ convex function (with $f(1) = 0$), letting $\lambda^{*} = \sup_{P' \in \mathscr{P}(\mathcal{X})}\bracket{W_c(P',P_G)/D_f(P',P_G)}$, we have for all $\lambda \geq \lambda^{*}$
\begin{align*}
\operatorname{GAN}_{\lambda f}(P_X,G; \mathcal{H}_c) = \overline{W}_{c, \lambda \cdot f}(P_X,G) = \operatorname{WAE}_{c,\lambda \cdot D_f}(P_X,G) = W_c(P_X,P_G).
\end{align*}
\end{theorem}
(Proof in Appendix, Section \ref{proof:final-equality-condition}). Note that Theorem \ref{final-equality-condition} holds for any $f$ (satisfying properties of the Theorem) and so one can estimate the Wasserstein distance using any $f$ as long as $\lambda$ is scaled to $\lambda^{*}$. In order to understand the quantity
\begin{align*}
\lambda^{*} = \sup_{P' \in \mathscr{P}(\mathcal{X})}\bracket{W_c(P',P_G)/D_f(P',P_G)},
\end{align*}
there are two extremes in which the supremum may be unbounded. The first case is when $P'$ is taken far from $P_G$ so that $W_c(P',P_G)$ increases, however one should note that in the case when $\Delta = \max_{x, x' \in \mathcal{X}} c(x,x') < \infty$ then $W_c \in [0,\Delta]$ and so $W_c$ will be finite whereas $D_f(P',P_G)$ can possibly diverge to $\infty$, making $\lambda^{*} \to 0$. The other case is when $P'$ is made close to $P_G$, in which case $\frac{1}{D_f(P',P_G)} \to \infty$ however $W_c(P',P_G) \to 0$ so the quantity $\lambda^{*}$ can still be small in this case, depending on the rate of decrease between $W_c$ and $D_f$. Now suppose that $f(x) = \card{x-1}$ and $c(x,y) = \delta_{x-y}$, in which case $D_f = W_c$ and thus $\lambda^{*} = 1$. In this case, Theorem \ref{final-equality-condition} reduces to the standard result regarding the equivalence between Wasserstein distance and $f$-divergence intersecting at the variational divergence under these conditions.
\section{Generalization bounds}
\label{sec:gen-bounds}
We prove generalization bounds using machinery developed in \citep{weed2017sharp} and thus introduce their definitions and notations.
\begin{definition}
For a set $S \subseteq \mathcal{X}$, we denote $N_{\eta}(S)$ to be the $\eta$-covering number of $S$, which is the smallest $m \in \mathbb{N}_{*}$ such that there exists closed balls $B_1,\ldots,B_m$ of radius $\eta$ with $S \subseteq \bigcup_{i=1}^m B_i$.
\end{definition}
\begin{definition}
For any $P \in \mathscr{P}(\mathcal{X})$, the $(\eta, \tau)$-covering number is
\begin{align*}
N_{\eta}(P,\tau) = \inf\braces{N_{\eta}(S) : P(S) \geq 1 - \tau},
\end{align*}
and the $(\eta, \tau)$-dimension is
\begin{align*}
d_{\eta}(P, \tau) = \frac{\log N_{\eta}(P, \tau) }{-\log \eta}.
\end{align*}
\end{definition}
\begin{definition}
The $1$-Upper Wasserstein dimension of $P \in \mathscr{P}(\mathcal{X})$ is
\begin{align*}
d^{*}(P) = \inf\braces{s \in (2, \infty) : \limsup_{\eta \to 0} d_{\eta} (P, \eta^{\frac{s}{s -  2}}) \leq s }.
\end{align*}
\end{definition}
We make an assumption of $P_X$ and $P_G$ having bounded support to achieve the following bounds. For any $P \in \mathscr{P}(\mathcal{X})$ in a metric space $(\mathcal{X}, c)$, we use define $\Delta_{P,c} = \text{diam}_c (\text{supp}(P))$. We are now ready to present the generalization bounds.
\begin{theorem}
  \label{main-gen-result}
  Let $(\mathcal{X}, c)$ be a metric space and suppose $\Delta := \max\braces{\Delta_{c,P_X}, \Delta_{c,P_G}} < \infty$. For any $n \in \mathbb{N}_{*}$, let $\hat{P}_X$ and $\hat{P}_G$ denote the empirical distribution with $n$ samples drawn i.i.d from $P_X$ and $P_G$ respectively. Let $s_X > d^{*}(P_X)$ and $s_G > d^{*}(P_G)$. For all $f:\mathbb{R} \to (-\infty,\infty]$ convex functions, $f(1) = 0$ and $\lambda > 0$, we have
  \begin{align}
  \operatorname{GAN}_{\lambda f}(P_X,G; \mathcal{H}_c)  \leq \overline{W}_{c,\lambda \cdot f}(\hat{P}_X, P_G) +  O\bracket{n^{-1/s_X} + \Delta\sqrt{\frac{1}{n}\ln\bracket{\frac{1}{\delta}}}} , \label{main:first}
  \end{align}
  with probability at least $1 - \delta$ for any $\delta \in (0,1)$ and if $f(x) = \card{x-1}$ is chosen then we have for all $\lambda > 0$
  \begin{align}
    \operatorname{GAN}_{\lambda f}(P_X,G; \mathcal{H}_c)  \leq \overline{W}_{c,\lambda \cdot f}(\hat{P}_X, \hat{P}_G)+ O\bracket{n^{-1/s_X} + n^{-1/s_G} + \Delta \sqrt{\frac{1}{n}\ln\bracket{\frac{1}{\delta}}}}, \label{main:first}
  \end{align}
  with probability at least $1 - \delta$ for any $\delta \in (0,1)$.
\end{theorem}
(Proof in Appendix, Section \ref{proof:main-gen-result}). First note that there is no requirement on $G$ to be invertible and no restriction on $\lambda$. Second, there are the quantities $s_X$,$s_G$ and $\Delta$ that are influenced by the distributions $P_X$ and $P_G$. If $G$ is invertible in the above then the left hand side of both bounds becomes $\overline{W}_{c,\lambda \cdot f}(P_X,G)$ by Theorem \ref{fgan-wae-equivalence}. One could suspect that $\overline{W}_{c,\lambda \cdot f}(\hat{P}_X, \hat{P}_G)$ may be unbounded by drawing parallels to $f$-divergences, in which case may be unbounded. However this is not the case since
\begin{align*}
\overline{W}_{c,\lambda \cdot f}(\hat{P}_X, \hat{P}_G) &\leq \inf_{E \in \mathscr{F}(\mathcal{X}, \mathscr{P}(\mathcal{X}) )} \braces{W_c((G \circ E) \# P_X,P_X) + \lambda D_f(E\#\hat{P}_X , \hat{P}_Z)},
\end{align*}
and since we search $E \in \mathscr{F}(\mathcal{X}, \mathscr{P}(\mathcal{Z}))$ and there exists a $E'$ such that $E' \#\hat{P}_X$ shares the support of $\hat{P}_Z$, in which case will result in a bounded value. Using Theorem \ref{final-equality-condition}, one can set $\lambda$ large enough so that the expressions in the bound can become Wasserstein distances.

We show now that the $\overline{W}$ can be upper bounded by the Wasserstein distance. Consider the entropy regularized Wasserstein distance:
\begin{align*}
W_{c, \epsilon}(P_X,G) := \inf_{\pi \in \Pi(P_X,P_G)} \braces{ \int_{\mathcal{X} \times \mathcal{X}} c(x,y) \pi(x,y) + \epsilon \cdot \text{KL}(\pi, P_X \otimes P_G)  },
\end{align*}
we have the following.
\begin{theorem}
\label{fwae-entropy}
For any $c: \mathcal{X} \times \mathcal{X} \to \mathbb{R}$, $\lambda > 0$ and $f: \mathbb{R} \to (-\infty, \infty]$ convex function (with $f(1) = 0$) we have
\begin{align}
\overline{W}_{c, \lambda \cdot f}(P_X,G) \leq W_c(P_X,P_G) \leq W_{c,\epsilon}(P_X,G). \label{entropy-gap}
\end{align}
\end{theorem}
(Proof in Appendix, Section \ref{proof:fwae-entropy}). Since our goal is to minimize $\overline{W}$, we can minimize the upper bounds. Since the above objectives are Wasserstein distances, we can make use of the existing efficient solvers for these quantities. Indeed, majority of these solvers are concerned with discrete problems, which is presented in Theorem \ref{main-gen-result}: $\overline{W}_{c,\lambda \cdot f}(\hat{P}_X, \hat{P}_G)$. Noting also that from Theorem \ref{final-equality-condition}, setting $\lambda$ to be higher values closes the gap in (\ref{entropy-gap}).

\section{Discussion and Conclusion}
This work is the first to prove a generalized primal-dual betweenship between GANs and Autoencoders. Consequently, this results in the elucidation for the close performance between WAE and $f$-GANs. Furthermore, we explored the effect of weighting the reconstruction and regularization on the WAE objective, showing relationships to both $f$-divergences and Wasserstein metrics along with the impact on the duality relationship. This equivalence allows us to prove generalization results, which to the best of our knowledge, are the first bounds given for Autoencoder models. Furthermore, using connections to the Wasserstein metrics, we can employ efficient (regularized) OT solvers to approximate upper bounds on the generalization bounds, which involve discrete distributions and thus are natural for such solvers.

The consequences of unifying two paradigms are plentiful, generalization bounds being an example. One line of extending and continuing this line of work can explore the case when using a general cost $c$ (as opposed to a metric), invoking the generalized Wasserstein dual in the goal of forming a generalized GAN. Our paper provides a basis to unify Adversarial Networks and Autoencoders through a primal-dual relationship, and open doors for the further unification of related models.


\acks{We would like to acknowledge anonymous reviewers and the Australian Research Council of Data61.}

\bibliography{references}
\newpage
\appendix
\section{Appendix}
\subsection{Proof of Theorem \ref{fgan-wae-equivalence}}
\label{proof:fgan-wae-equivalence}
In order to prove the theorem, we make use of the dual form of the restricted variational form of an $f$-divergence:
\begin{theorem}[\citep{liu2018inductive}, Theorem 3]
\label{liu-theorem}
Let $f:\mathbb{R} \to (-\infty,\infty]$ denote a convex function with property $f(1) = 0$ and suppose $H$ is a convex subset of $\mathscr{F}(\mathcal{X},\mathbb{R})$ with the property that for any $h \in H$ and $b \in \mathbb{R}$, we have $h+b \in H$. Then for any $P,Q \in \mathscr{P}(\mathcal{X})$ we have
\begin{align*}
    \sup_{h \in H}\braces{\E_{x \sim P}[h(x)] - \E_{x \sim Q}[f^{*}(h(x))] } = \inf_{P' \in \mathscr{P}(\mathcal{X})} \braces{D_f(P',Q) + \sup_{h \in H}\braces{\E_{P}[h(x)] - E_{P'}[h(x)]} }
\end{align*}
\end{theorem}
The goal is now to set $H = \mathcal{H}_c$ however there are some conditions of the above that we require
\begin{lemma}
\label{hc-regular}
If $c$ is a metric then $\mathcal{H}_c$ is convex and closed under addition.
\end{lemma}
\begin{proof}
Let $f \in \mathcal{H}_c$ and consider define $h = f + b$ for some $b \in \mathbb{R}$, we then have
\begin{align*}
\card{h(x) - h(y)} &= \card{f(x) + b - f(y) - b}\\
                   &= \card{f(x) - f(y)}\\
                   &\leq  c(x,y)
\end{align*}
Consider some $\lambda \in [0,1]$ and set $h(x) = \lambda \cdot f(x) + (1 - \lambda) \cdot g(x)$ for some $f,g \in \mathcal{H}_c$. We then have
\begin{align*}
\card{h(x) - h(y)} &= \card{\lambda \cdot f(x) + (1 - \lambda) \cdot g(x) -   \lambda \cdot f(y) - (1 - \lambda) \cdot g(y)}\\
                   &= \card{\lambda \cdot (f(x) - f(y)) + (1 - \lambda) \cdot (g(x) - g(y))}\\
                   &\leq \lambda \cdot \card{f(x) - f(y) } + (1 - \lambda) \cdot \card{g(x) - g(y)}\\
                   &\leq \lambda \cdot   c(x,y) + (1 - \lambda)  \cdot c(x,y)\\
                   &=   c(x,y)
\end{align*}
for all $x,y \in \mathcal{X}$.
\end{proof}
We require a lemma regarding the decomposibility of $G$ for $f$-divergences.
\begin{lemma}
\label{fgan-wae:two}
Let $G: \mathcal{Z} \to \mathcal{X}$ and let $P,Q$ be two distributions over $\mathcal{Z}$. We have that \begin{align*} D_f (G \# P, G\# Q) \leq D_f(P,Q), \end{align*} with equality if $G$ is invertible. Furthermore, if $f$ is differentiable then we have equality for a weaker condition: for any $z,z' \in \mathcal{Z}, G(z) = G(z') \implies f'(\frac{dP}{dQ}(z)) = f'(\frac{dP}{dQ}(z'))$.
\end{lemma}
\begin{proof}
By writing the variational form from \citep{nguyen2010estimating} (Lemma 1), we have
\begin{align*}
D_f(G \# P, G \# Q) &= \sup_{h \in \mathscr{F}(\mathcal{X}, \mathbb{R})}\braces{\E_{x \sim G\#P} [h(x)] - \E_{x \sim G \# Q}[ f^{*}(h(x))]}\\
		            &= \sup_{h \in  \mathscr{F}(\mathcal{X}, \mathbb{R})}\braces{\E_{z \sim P}[h(G(z))] - \E_{z \sim Q}[f^{*}(h(G(z)))]}\\
		           &= \sup_{h \in \mathscr{F}(\mathcal{X}, \mathbb{R}) \circ G  } \braces{\E_{z \sim P}[h(z)] - \E_{z \sim Q}[ f^{*}(h(z))]}\\
			&\leq \sup_{h \in \mathscr{F}(\mathcal{Z}, \mathbb{R})} \braces{\E_{z \sim P}[h(z)] - \E_{z \sim Q}[ f^{*}(h(z))]}\\
			&= D_f(P,Q),
\end{align*}
where we used the fact that $\mathscr{F}(\mathcal{X}, \mathbb{R}) \circ G \subseteq \mathscr{F}(\mathcal{Z}, \mathbb{R})$. If $G$ is invertible then we applying the above with $G \gets G^{-1}$, $P \gets G \# P$ and $Q \gets G \# Q$, we have $$D_f(G^{-1} \# (G \# P), G^{-1} \# (G \# Q)) \leq D_f(G \# P, G \# Q),$$ which is just the reverse direction $D_f(P, Q) \leq D_f( G\# P, G \#Q)$, and so equality holds. Suppose now that $f$ is differentiable then note that inequality holds when $f'(dP/dQ) \in \mathscr{F}(\mathcal{X}, \mathbb{R}) \circ G$ (See proof of Lemma 1 in \citep{nguyen2010estimating}), which is equivalent to asking if there exists a function $\varphi_f \in \mathscr{F}(\mathcal{X}, \mathbb{R})$ such that
\begin{align*}
\varphi_f \circ G = f'\bracket{\frac{dP}{dQ}}.
\end{align*}
For any $z \in \mathcal{Z}$, we can construct $\varphi_f$ to map $G(z)$ to $f'\bracket{\frac{dP}{dQ}}(z)$ and due to the condition in the lemma, we can guarantee $\varphi_f$ will indeed be a function and thus exists.
\end{proof}
We need a Lemma that will allow us to upper bound the Wasserstein distance.
\begin{lemma}
  \label{fgan-wae:three}
For any $E \in \mathscr{F}(\mathcal{X}, \mathscr{P}(\mathcal{Z}))$, $G \in \mathscr{F}(\mathcal{Z}, \mathcal{X})$ and $c:\mathcal{X} \times \mathcal{X} \to \mathbb{R}$, we have
\begin{align*}
W_c((G \circ E) \# P_X,P_X) \leq  \int_{\mathcal{X}} \E_{z \sim E(x)}[c(x,G(z))] dP_X(x).
\end{align*}
\end{lemma}
\begin{proof}
We quote a reparametrization result from \citep{tolstikhin2017wasserstein} Theorem 1 that if $G$ is deterministic then the Wasserstein distance can be reparametrized as
\begin{align}
W_c( G \# (E \# P_X), P_X) &= \inf_{Q \in \mathscr{F}(\mathcal{X}, \mathscr{P}(\mathcal{Z})) : Q\# P_X = E \# P_X}\int_{\mathcal{X}} \E_{z \sim Q(x)}[c(x,G(z))] dP_X(x) \label{first-wae-use} \\
				 &\leq \int_{\mathcal{X}} \E_{z \sim E(x)}[c(x,G(z))] dP_X(x). \nonumber
\end{align}
\end{proof}
We also need a Lemma regarding the relationship between $\overline{W}$ and $\operatorname{WAE}$.
\begin{lemma}
  \label{f-WAEisWAE}
Let $f:\mathbb{R} \to (-\infty,\infty]$ be a convex function with $f(1) = 0$, then we have
\begin{align*}
\overline{W}_{c, \lambda \cdot f}(P_X,G) \leq \operatorname{WAE}_{c,\lambda \cdot D_f}(P_X,G).
\end{align*}
\end{lemma}
\begin{proof}
Consider the optimal encoder $E^{*}$ from the $f$-WAE objective. Let $Q^{*} = E^{*} \# P_X$. We then have that
\begin{align*}
\overline{W}_{c, \lambda \cdot f}(P_X,G) = W_c(P_X, G \# Q^{*}) + \lambda \cdot D_f(Q^{*}, P_Z).
\end{align*}
Let $\pi \in \Pi( P_X , E \# Q^{*} )$ be the optimal coupling under the metric $c$. By the Gluing lemma \citep{vOT}, one can construct a triple $(X,Y,Z)$ where $(X,Y) \sim \pi$, $Z \sim Q^{*}$ and $Y = G(Z)$ almost surely. Let $\pi'$ be the distribution over $(Y,Z)$ and consider the conditional distribution over $Z$ given $Y$, associated with $E_{\pi'} \in \mathscr{F}(\mathcal{X}, \mathscr{P}(\mathcal{Z}))$. We have $E_{\pi'}\# P_X = Q^{*}$ and so we have
\begin{align*}
\operatorname{WAE}_{c,\lambda \cdot D_f}(P_X,G) &\leq \int_{\mathcal{X}} \E_{z \sim E_{\pi'}(y)} [c(x,G(z)) ] dP_X + D_f(E_{\pi'} \# P_X, P_Z)\\
                                                  &= \int_{\mathcal{X}} \E_{z \sim E_{\pi'}(y)} [c(x,G(z)) ] dP_X + D_f(Q^{*} , P_Z)\\
                                                  &= \int_{\mathcal{X} \times \mathcal{X}}  [c(x,y)] d\pi'(x,y) + D_f(Q^{*} , P_Z)\\
                                                  &= W_c(P_X, G \# Q^{*}) + \lambda \cdot D_f(Q^{*}, P_Z).\\
                                                  &= \overline{W}_{c, \lambda \cdot f}(P_X,G).
\end{align*}
\end{proof}
Finally, we need a lemma to justify reparametrizations.
\begin{lemma}
\label{reparamP-lem}
If $G: \mathcal{Z} \to \mathcal{X}$ is invertible then for any $P' \in \mathscr{P}(\mathcal{X})$ such that $P' \ll P_G$, then there exists an $E \in \mathscr{F}\bracket{\mathcal{X}, \mathscr{P}(\mathcal{Z})}$ such that $P' = G \# E \# P_X$.
\end{lemma}
\begin{proof}
From the assumption, we have $\text{Supp}(P’) \subseteq \text{Supp}(P_G) \subseteq \text{Im}(G)$ and so by invertibility of $G$, we can set $Q = G^{-1} \# P’$ and construct a conditional distribution $E$ (between marginals $Q$ and $P_X$) to get $Q = E \# P_X$, hence $P’ = G \# E \# P_X$.
\end{proof}
We are now ready to prove the theorem. Set $H =\mathcal{H}_c$ (the set of $1$-Lipschitz functions) and note that $\lambda f$ is a convex function satisfying $\lambda f(1) = 0$ and so substituting $f \gets \lambda f$, we get that $D_{\lambda f}(\cdot, \cdot) = \lambda D_{f} (\cdot, \cdot)$. Hence, we have
\begin{align}
\text{GAN}_{\lambda f}(P_X,G; \mathcal{H}_c) &= \sup_{h \in H_c}\braces{\E_{x \sim P_X}[h(x)] - \E_{x \sim P_G}[(\lambda f)^{\star}(h(x))]} \nonumber \\
																						 &= \inf_{P' \in \mathscr{P}(\mathcal{X})} \braces{\lambda D_f (P',P_G) + W_c (P',P_X)} \nonumber \\
                                             &= \inf_{P' \in \mathscr{P}(\mathcal{X}) :  P' << P_g} \braces{\lambda D_f (P',P_G) + W_c (P',P_X)} \nonumber \\
																						 &= \inf_{E \in \mathscr{F}\bracket{\mathcal{X}, \mathscr{P}(\mathcal{Z})}} \braces{\lambda D_f((G \circ E) \# P_X, G\# P_Z) + W_c((G \circ E) \# P_X, P_X)} \label{reparam-P} \\
																						 &\stackrel{(*)}{\leq} \inf_{E \in \mathscr{F}\bracket{\mathcal{X}, \mathscr{P}(\mathcal{Z})}} \braces{\lambda D_f(E \# P_X,  P_Z) + W_c((G \circ E) \# P_X, P_X)} \nonumber\\
                                             &= \overline{W}_{c,\lambda \cdot f}(P_X,G) \nonumber \\
                                             &\leq \inf_{E \in \mathscr{F}(\mathcal{X}, \mathscr{P}(\mathcal{Z}))} \braces{ \int_{\mathcal{X}} \E_{z \sim E(x)}[c(x,G(z))] dP_X(x) + \lambda D_f (E \# P_X,  P_Z) } \nonumber \\
                                             &= \text{WAE}_{c,\lambda \cdot D_f}(P_X,G), \label{first-ineq}
\end{align}
where $(\ref{reparam-P})$ is an equality when $G$ is invertible from Lemma \ref{reparamP-lem} and $(*)$ is $=$ if $G$ satisfies the requirement of Lemma \ref{fgan-wae:two}. To prove the final inequality, note that if $E^{*}$ satisfies the condition of the Theorem then
\begin{align}
\overline{W}_{c,\lambda \cdot f}(P_X,G) &=  W_c((G \circ E^{*}) \# P_X, P_X) + \lambda D_f(E^{*} \# P_X,  P_Z) \nonumber \\
                                          &=  W_c( G \# (E^{*} \# P_X), P_X) \nonumber \\
                                          &= W_c(P_G,P_X). \label{fwae-wc}
\end{align}
Next, notice that
\begin{align}
&\text{WAE}_{c,\lambda \cdot D_f}(P_X,G)\nonumber \\ &= \inf_{E \in \mathscr{F}(\mathcal{X}, \mathscr{P}(\mathcal{Z}))} \braces{ \int_{\mathcal{X}} \E_{z \sim E(x)}[c(x,G(z))] dP_X(x) + \lambda D_f (E \# P_X,  P_Z) } \nonumber \\
                                          &\leq \inf_{E \in \mathscr{F}(\mathcal{X}, \mathscr{P}(\mathcal{Z})) : E \# P_X = P_Z  } \braces{ \int_{\mathcal{X}} \E_{z \sim E(x)}[c(x,G(z))] dP_X(x) + \lambda D_f (E \# P_X,  P_Z) } \nonumber \\
                                          &\leq \inf_{E \in \mathscr{F}(\mathcal{X}, \mathscr{P}(\mathcal{Z})) : E \# P_X = P_Z  } \braces{ \int_{\mathcal{X}} \E_{z \sim E(x)}[c(x,G(z))] dP_X(x)  } \nonumber \\
                                          &= W_c(P_X,P_G)\label{wae-theorem-use}\\
                                          &= \overline{W}_{c,\lambda \cdot f}(P_X,G),
\end{align}
where (\ref{wae-theorem-use}) follows from the reparametrized Wasserstein distance from \citep{tolstikhin2017wasserstein} (Theorem 1), which we used in (\ref{first-wae-use}) and the final step follows from (\ref{fwae-wc}). Combining $\text{WAE}_{c,\lambda \cdot D_f}(P_X,G) \leq \overline{W}_{c,\lambda \cdot f}(P_X,G)$ with $\text{WAE}_{c,\lambda \cdot D_f}(P_X,G) \geq \overline{W}_{c,\lambda \cdot f}(P_X,G)$ (from \ref{first-ineq}) yields equality and concludes the proof.
\subsection{Proof of Theorem \ref{main-gen-result}}
\label{proof:main-gen-result}
We first prove a lemma that will apply to both cases. Recalling that for any metric space $(\mathcal{X}, c)$ and $P \in \mathscr{P}(\mathcal{X})$ we define $\Delta_{P,c} = \text{diam}_c (\text{supp}(P))$.
\begin{lemma}
  \label{wasserstein-convergence}
Let $(\mathcal{X}, c)$ be a metric space. For any $P \in \mathscr{P}(\mathcal{X})$, suppose $\Delta_{P,c}  < \infty$ and let $\hat{P}$ denote the empirical distribution after drawing $n$ i.i.d samples for some $n \in \mathbb{N}_{*}$. If $s > d^{*}(P)$, then we have
\begin{align*}
\operatorname{IPM}_{\mathcal{H}_c}(P,\hat{P}) \leq O(n^{-1/s}) + \frac{\Delta_{P,c}}{2} \sqrt{\frac{2}{n}\ln\bracket{\frac{1}{\delta}}}
\end{align*}
\end{lemma}
\begin{proof}
  We appeal to McDiarmind's Inequality and use a standard method, as shown in \citep{bartlett2002rademacher}, to bound the quantity.
  \begin{theorem}[McDiarmind's Inequality]
  Let $X_1,\ldots,X_n$ be $n$ independent random variables and consider a function $\Phi: \mathcal{X}^n \to \mathbb{R}$ such that there exists constants $c_i >0 $ (for $i = 1,\ldots,n$) with
  \begin{align*}
  \sup_{x_1,\ldots,x_n,x'_i} \card{\Phi(x_1,\ldots,x_n) - \Phi(x_1,\ldots,x_{i-1}, x'_i, x_{i+1}, \ldots, x_n )} \leq c_i.
\end{align*}
  Then for any $t > 0$, we have
  \begin{align*}
  \operatorname{Pr}\left [  \Phi(X_1,\ldots,X_n) - \E\left [ \Phi(X_1,\ldots,X_n) \right] \geq t \right] \leq \exp\bracket{\frac{-2t^2 }{\sum_{i=1}^n c_i^2}}
\end{align*}
  \end{theorem}
  Let $\mathcal{F} = \mathcal{H}_{c}$ then let
  \begin{align*}
  \Phi(S) = \text{IPM}_{\mathcal{H}_{c}}(P, \hat{P}).
\end{align*}
  Noting that
  \begin{align*}
  \card{\Phi(x_1,\ldots,x_n) - \Phi(x_1,\ldots,x_{i-1}, x'_i, x_{i+1}, \ldots, x_n )} &\leq \frac{1}{n}\card{f(x_i) - f(x'_i)}\\
                                                                      &\leq \frac{1}{n} \cdot c(x_i,x'_i)\\
  																																		&\leq \frac{ \Delta_{P,c}}{n},
  \end{align*}
  where the first inequality follows as each $f$ is $1$-Lipschitz and the second follows from the fact that each $x,x' \in \operatorname{supp}(P)$. This allows us to set $c_i = \Delta /n$ for all $i = 1,\ldots, n$. Now applying McDiarmind's inequality with $t =  \Delta_{P,c}/2 \sqrt{\frac{2}{n}\ln\bracket{\frac{1}{\delta}}}$ yields (for a sample $S \sim P^n$)
  \begin{align*}
  &\text{Pr}\left[ \Phi(S) - \E\Phi(S) \geq  \frac{\Delta_{P,c}}{2} \sqrt{\frac{2}{n}\ln\bracket{\frac{1}{\delta}}} \right] \leq \delta\\
  &\text{Pr}\left[ \Phi(S) - \E\Phi(S) \leq \frac{\Delta_{P,c}}{2} \sqrt{\frac{2}{n}\ln\bracket{\frac{1}{\delta}}} \right] \geq 1 - \delta,
\end{align*}
  and thus
  \begin{align*}
  \Phi(S) \leq \E\Phi(S) + \frac{\Delta_{P,c}}{2} \sqrt{\frac{2}{n}\ln\bracket{\frac{1}{\delta}}}.
\end{align*}
  Noting that $\E \Phi(S) = \E[W_c(P,\hat{P})]$ (from Lemma \ref{w-duality}), we appeal to a case of Theorem 1 in \citep{weed2017sharp} where $p = 1$, which tells us that if $s > d^{*}(P)$ then $\E[W_c(P,\hat{P})] = O(n^{-1/s})$. Since this is the requirement in the lemma, the proof concludes.
\end{proof}
We will make use of this lemma for both $P_X$ and $P_G$ and use $\Delta$ for both cases since $\Delta \geq \Delta_{P_X,c}$ and $\Delta \geq \Delta_{P_G,c}$.
For the general case of any $f$, let (abusing notation) $G = \text{GAN}_{\lambda f}(P_X,G; \mathcal{H}_c)$ and $\hat{G}$ denote the empirical counterpart with $n$ samples, and let $h^{1}, h^{2} \in \mathcal{H}_c$ denote their witness functions. We then have
\begin{align*}
&G - \hat{G}\\ &= \sup_{h \in \mathcal{H}_c}\braces{\E_{x \sim P_X}[h(x)] - \E_{x \sim P_G}[(\lambda f)^{\star}(h(x))] } - \sup_{h \in \mathcal{H}_c}\braces{\E_{x \sim \hat{P}_X}[h(x)] - \E_{x \sim P_G}[(\lambda f)^{\star}(h(x))] }\\
					  &= \E_{x \sim P_X}[h^{1}(x)] - \E_{x \sim P_G}[(\lambda f)^{\star}(h^{1}(x))] - \E_{x \sim \hat{P}_X}[h^{2}(x)] + \E_{x \sim P_G}[(\lambda f)^{\star}(h^{2}(x))]\\
						&\leq \E_{x \sim P_X}[h^{1}(x)] - \E_{x \sim \hat{P}_X}[h^{1}(x)] + \E_{x \sim P_G}[(\lambda f)^{\star}(h^{1}(x))] - \E_{x \sim P_G}[(\lambda f)^{\star}(h^{1}(x))]\\
            &= \E_{x \sim P_X}[h^{1}(x)] - \E_{x \sim \hat{P}_X}[h^{1}(x)]\\
						&\leq \sup_{h \in \mathcal{H}_c}\braces{ \E_{x \sim P_X}[h(x)] - \E_{x \sim \hat{P}_X}[h(x)]}\\
						&= \text{IPM}_{\mathcal{H}_c}(P_X,\hat{P}_X)\\
            &\leq O(n^{-1/s_X}) + \frac{\Delta}{2} \sqrt{\frac{2}{n}\ln\bracket{\frac{1}{\delta}}},
\end{align*}
where the last step is an application of Lemma \ref{wasserstein-convergence}. Applying Theorem \ref{fgan-wae-equivalence}, we get $\hat{G} \leq \overline{W}_{c, \lambda \cdot f}$ and rearrangement of the above shows the first bound. For the case of $f(x) = \card{x-1}$, note that if $\mathcal{F} \subseteq \mathscr{F}(\mathcal{X}, \mathbb{R})$ is such that $-\mathcal{F} = \mathcal{F}$, then $\text{IPM}_{\mathcal{F}}$ is a pseudo-metric and satisfies the triangle inequality, which allows us to have
\begin{align}
\text{IPM}_{\mathcal{F}}(P_X, P_G) &\leq \text{IPM}_{\mathcal{F}}(P_X, \hat{P}_X) + \text{IPM}_{\mathcal{F}}(\hat{P}_X, P_G) \nonumber \\
																	 &\leq \text{IPM}_{\mathcal{F}}(P_X, \hat{P}_X) + \text{IPM}_{\mathcal{F}}(P_G, \hat{P}_G) + \text{IPM}_{\mathcal{F}}(\hat{P}_X, \hat{P}_G). \label{first-gen:triangle}
\end{align}
Next, we set $\mathcal{F} = \mathcal{F}_{c,\lambda}$, and noting that $\mathcal{F}_{c, \lambda} \subseteq \mathcal{H}_c$, we have
\begin{align}
  \text{IPM}_{\mathcal{F}_{c,\lambda}}(P_X, P_G) &\leq \text{IPM}_{\mathcal{F}_{c,\lambda}}(P_X, \hat{P}_X) + \text{IPM}_{\mathcal{F}_{c,\lambda}}(P_G, \hat{P}_G) + \text{IPM}_{\mathcal{F}_{c,\lambda}}(\hat{P}_X, \hat{P}_G) \nonumber \\
                                                &\leq \text{IPM}_{\mathcal{H}_{c}}(P_X, \hat{P}_X) + \text{IPM}_{\mathcal{H}_{c}}(P_G, \hat{P}_G) + \text{IPM}_{\mathcal{H}_{c}}(\hat{P}_X, \hat{P}_G) \nonumber \\
                                                &\leq \text{IPM}_{\mathcal{H}_{c}}(\hat{P}_X, \hat{P}_G) + O(n^{-1/s_X} + n^{-1/s_G}) + \Delta \sqrt{\frac{2}{n}\ln\bracket{\frac{2}{\delta}}}, \label{vardiv-upper}
\end{align}
where the final inequality is an application of Lemma \ref{wasserstein-convergence} like before. However since we use McDiarmind's inequality twice, we set $\delta \gets \delta/2$ and use union bound to have the above inequality with probability $1 - \delta$. The final step is to note that when $f(x) = \card{x-1}$ then for any $\lambda > 0$,
\begin{align*}
(\lambda f)^{\star}(x) = \begin{cases} x & x \leq \lambda \\ \infty & x > \lambda \end{cases}
\end{align*}
and so we have
\begin{align*}
\text{GAN}_{\lambda f}(P_X,G; \mathcal{H}_c) &= \sup_{h \in \mathcal{H}_c} \braces{\E_{x \sim P_X}[h(x)] - \E_{x \sim P_G}[(\lambda f)^{\star}(h(x))]}\\
																						 &= \sup_{h \in \mathcal{H}_c : \card{h} \leq \lambda} \braces{\E_{x \sim P_X}[h(x)] - \E_{x \sim P_G}[h(x)]}\\
																						 &= \sup_{h \in \mathcal{F}_{c,\lambda}} \braces{\E_{x \sim P_X}[h(x)] - \E_{x \sim P_G}[h(x)]}\\
																						 &= \text{IPM}_{\mathcal{F}_{c,\lambda}}(P_X,P_G).
\end{align*}
By Theorem \ref{fgan-wae-equivalence}, we have $\text{IPM}_{\mathcal{F}_{c,\lambda}}(\hat{P}_X, \hat{P}_G) = \text{GAN}_{\lambda f}(\hat{P}_X,G; \mathcal{H}_c) \leq \overline{W}_{c,\lambda \cdot f}(\hat{P}_X, G)$ where $\text{GAN}_{\lambda f}(\hat{P}_X,G; \mathcal{H}_c)$ is the objective with $\hat{P}_X$ and $\hat{P}_G$. Putting this together with (\ref{vardiv-upper}), we get
\begin{align*}
\text{GAN}_{\lambda f}(P_X,G; \mathcal{H}_c) &= \text{IPM}_{\mathcal{F}_{c,\lambda}}(P_X,P_G)\\
                                             &\leq \text{IPM}_{\mathcal{H}_{c}}(\hat{P}_X, \hat{P}_G) + O(n^{-1/s}) + \Delta \sqrt{\frac{2}{n}\ln\bracket{\frac{1}{\delta}}}\\
                                             &= \text{GAN}_{\lambda f}(\hat{P}_X,G; \mathcal{H}_c)+ O(n^{-1/s}) + \Delta \sqrt{\frac{2}{n}\ln\bracket{\frac{1}{\delta}}}\\
                                             &\leq \overline{W}_{c,\lambda \cdot f}(\hat{P}_X, G)+ O(n^{-1/s_X} + n^{-1/s_G}) + \Delta \sqrt{\frac{2}{n}\ln\bracket{\frac{2}{\delta}}}.
\end{align*}
\subsection{Proof of Theorem \ref{fdiv-fwae}}
\label{proof:fdiv-fwae}
First, using Theorem \ref{fgan-wae-equivalence} and the fact that the $f$-GAN objective is a lower bound to $D_f$, we have that
\begin{align*}
\overline{W}_{\gamma^{*} \cdot c, f}(P_X,G) &=\text{GAN}_{f}(P_X,G,\mathcal{H}_{\gamma^{*} c})\\
					                 &\leq D_f.
\end{align*}
It is well known that $f'(dP_X/dP_G)$ is the maximizer of $L(h) = \E_{x \sim P_X}[h(x)] - \E_{x \sim P_G}[f^{\star}(h(x))]$, and so the proof concludes by showing that $f'(dP_X/dP_G) \in \mathcal{H}_{\gamma^{*} c}$. Note that $h \in \mathcal{H}_{\gamma^{*} c}$ if and only if for all $x,x' \in \mathcal{X}, x \neq x'$
\begin{align*}
\card{h(x) - h(x')} \leq \gamma^{*},
\end{align*}
and so the $1$-Lipschitz functions are those that are bounded by their maximum and minimum value by $\gamma^{*}$. For any $x,x' \in \mathcal{X}, x \neq x'$ we have
\begin{align*}
\card{f'\bracket{\frac{dP_X}{dP_G}}(x) - f'\bracket{\frac{dP_X}{dP_G}}(x') } &\leq \card{f'\bracket{\frac{dP_X}{dP_G}}(x) - f'\bracket{0}}  \\
						&= \gamma^{*}.
\end{align*}
where the first inequality holds since $dP_X/dP_G > 0$ and that $f'$ is increasing as $f$ is convex. Hence $f'(dP_X/dP_G) \in \mathcal{H}_{\gamma^{*} c}$.

\subsection{Proof of Theorem \ref{final-equality-condition}}
\label{proof:final-equality-condition}
First note that
\begin{align*}
  \text{WAE}_{c,\lambda \cdot f}(P_X,P_G) &= \inf_{E \in \mathscr{F}(\mathcal{X}, \mathscr{P}(\mathcal{Z}))} \braces{\int_{\mathcal{X}} \E_{z \sim E(x)}[c(x,G(z))] dP_X(x) +  \lambda \cdot D_f (E \# P_X,  P_Z)}\\
                                          &\leq \inf_{E \in \mathscr{F}(\mathcal{X}, \mathscr{P}(\mathcal{Z})) : E \# P_X = P_Z} \braces{\int_{\mathcal{X}} \E_{z \sim E(x)}[c(x,G(z))] dP_X(x)}\\
                                          &= W_c(P_X,P_G),
\end{align*}
where the last equality holds from \citep{tolstikhin2017wasserstein} Theorem 1. Thus we have the chain of inequalities for all $\lambda$ and $f: \mathbb{R} \to (-\infty, \infty]$ (convex with $f(1) = 0$)
\begin{align*}
\operatorname{GAN}_{\lambda f}(P_X,G; \mathcal{H}_c) \leq \overline{W}_{c, \lambda \cdot}(P_X,P_G) \leq \text{WAE}_{c,\lambda \cdot f}(P_X,P_G) \leq W_c(P_X,P_G).
\end{align*}
We now show the opposite direction, which will conclude the proof.
\begin{lemma}
  \label{fgan-wc-lem}
For any metric $c$ and $f:\mathbb{R} \to (-\infty,\infty]$ convex function with $f(1) = 0$, if $$\lambda \geq \lambda^{*} = \sup_{P' \in \mathscr{P}(\mathcal{X})}\bracket{W_c(P',P_G)/D_f(P',P_G)},$$ then we have
\begin{align*}
\operatorname{GAN}_{\lambda f}(P_X,G; \mathcal{H}_c) \geq W_c(P_X,P_G)
\end{align*}
\end{lemma}
\begin{proof}
  First noting that $\lambda \geq \sup_{P' \in \mathscr{P}(\mathcal{X}) }\bracket{W_c(P',P_G)/D_f(P',P_G)}$, for all $P' \in \mathscr{P}(\mathcal{X})$, we have
  \begin{align*}
   \lambda D_f (P',  P_G) - W_c(P', P_G) \geq 0.
 \end{align*}
  Let $\tilde{\mathcal{Z}} = \mathcal{X}$,$\tilde{G} = \text{Id}$, $P_{\tilde{\mathcal{Z}}} = P_G$ and noting that $\tilde{G}$ is invertible, we can apply Theorem \ref{fgan-wae-equivalence} to get
  \begin{align*}
    \text{GAN}_{\lambda f}(P_X,G; \mathcal{H}_c) &= \overline{W}_{c,\lambda \cdot f}(P_X,\tilde{G} \# P_{\tilde{\mathcal{Z}}})\\
                                                 &= \inf_{E \in \mathscr{F}(\mathcal{X}, \mathcal{P}(\mathcal{X}))} \braces{W_c(E \# P_X, P_X) +  \lambda D_f (E \# P_X,  P_G)}\\
                                                 &\geq \inf_{E \in \mathscr{F}(\mathcal{X}, \mathcal{P}(\mathcal{X}))} \braces{W_c(P_X,P_G) - W_c(E \# P_X, P_G) + \lambda D_f (E \# P_X,  P_G)}\\
    																						 &\geq  \inf_{E \in \mathscr{F}(\mathcal{X}, \mathcal{P}(\mathcal{X}))} \braces{W_c(P_X,P_G)}\\
    																						 &= W_c(P_X,P_G).
  \end{align*}
\end{proof}

\subsection{Proof of Theorem \ref{fwae-entropy}}
\label{proof:fwae-entropy}
It is clear to see that $W_c \leq W_{c,\epsilon}$ with equality when $\epsilon = 0$ since $\text{KL}(\pi, P_X \otimes P_G) \geq 0$. Next, we have
\begin{align*}
\overline{W}_{c,\lambda \cdot f}(P_X,G) &= \inf_{E \in \mathscr{F}(\mathcal{X}, \mathscr{P}(\mathcal{Z}))} \braces{W_c(P_X, (G \circ E) \# P_X) +  \lambda D_f (E \# P_X,  P_Z)}\\
                                          &\leq  \inf_{E \in \mathscr{F}(\mathcal{X}, \mathscr{P}(\mathcal{Z})) : E \# P_X = P_Z} \braces{W_c(P_X, (G \circ E) \# P_X) +  \lambda D_f (E \# P_X,  P_Z)}\\
                                          &=  \inf_{E \in \mathscr{F}(\mathcal{X}, \mathscr{P}(\mathcal{Z})) : E \# P_X = P_Z} \braces{W_c(P_X, (G \circ E) \# P_X) }\\
                                          &=  \inf_{E \in \mathscr{F}(\mathcal{X}, \mathscr{P}(\mathcal{Z})) : E \# P_X = P_Z} \braces{W_c(P_X, P_G)}\\
                                          &= W_c(P_X,P_G)\\
                                          &\leq W_{c,\epsilon}.
\end{align*}

\subsection{Brenier Potentials and WAE}
\label{appendix-contribution}
In this section, we show that the Wasserstein Autoencoder under a certain parametrization corresponds to a well known result in Optimal Transport theory. Suppose
\begin{align*}
\Phi = \braces{\varphi \in \mathscr{F}(\mathcal{X}, \mathbb{R}) : \varphi \text{ or } -\varphi \text{ is convex } }.
\end{align*}
For some $d \geq 1$, let $\mathcal{X} \subset \mathbb{R}^d$ be a compact domain, $c(x,y) = \nrm{x- y}^2 / 2$ and consider some $P,Q \in \mathscr{P}(\mathcal{X})$, then the optimal transport map $\pi^{*} \in \mathscr{P}(\mathcal{X} \times \mathcal{X})$ to the problem
\begin{align*}
\inf_{\pi \in \mathscr{P} (\mathcal{X} \times \mathcal{X}) : \int_{\mathcal{X}}\pi(x,y) dy = P : \int_{\mathcal{X}} \pi(x,y) dx = Q} \braces{\int_{\mathcal{X} \times \mathcal{X}} c(x,y) d\pi(x,y) },
\end{align*}
is unique and supported on the graph $\braces{(x,T^{*}(x)) : x \in \mathcal{X}}$ where $T^{*} = \nabla \varphi^{*}$ and $\varphi$ is a convex function \citep{brenier1991polar}, referred to as the Brenier potential. Furthermore, the dual form of the above
\begin{theorem}[Pushforward and Pullback learning]
\label{pf-pb-learning}
Let $\mathcal{X} \subset \mathbb{R}^d$ be a compact domain and let $P,Q \in \mathscr{P}(\mathcal{X})$ be two probability distributions. Let $\mathcal{L}: \Phi \times \mathscr{F}\bracket{\mathcal{X}, \mathscr{P}(\mathcal{X})} \to \mathbb{R}$ be the bilinear objective:
\begin{align*}
\mathcal{L}_{P,Q}(\varphi, q) = \int_{\mathcal{X}} \E_{x \sim q(y)} [\varphi(x)] dQ(y)  - \int_{\mathcal{X}} \varphi(x) dP(x) - \int_{\mathcal{X}} \E_{x \sim q(y)} [x \cdot y] dQ(y).
\end{align*}
If $(\varphi^{*}, q^{*})$ are the optima of $\max_{\varphi \in \Phi} \inf_{q \in \mathscr{F}\bracket{\mathcal{X}, \mathscr{P}(\mathcal{X})}} \mathcal{L}_{P,Q}(\varphi, q)$ then we have that $\nabla \varphi \# P = Q$ and the mean function of $q^{*}$, given by $M(y) = \E_{x \sim q^{*}(y)}[x]$, matches $Q$ and $P$ under the moments of $\bracket{\mathscr{F} \circ \nabla \varphi}$.
\end{theorem}
(Proof in Section \ref{proof:pf-pb-learning}). We remark that by exploiting tools from Alexandrov Geometry \citep{lei2019geometric}, one can reparametrize the function using $\varphi_h(x) = \max_{i=1}^n \braces{x \cdot y_i + h(y_i)}$, where $\braces{y_i}_{i=1}^n$ are $n$ samples drawn $y_i \sim Q$ i.i.d, resulting in a convenient way to perform the optimization. We now show that this is related to the Wasserstein Autoencoder.
\begin{theorem}[Pullback learning and Convex IPM WAE]
\label{brenier-main}
Let $\mathcal{Z}$ be a latent space, $G: \mathcal{Z} \to \mathcal{X}$ and $G \circ \Phi = \braces{G \circ \varphi : \varphi \in \Phi} \subset \mathscr{F}(\mathcal{Z}, \mathbb{R})$. Let $P_X \in \mathscr{P}(\mathcal{X})$, $P_Z \in \mathscr{P}(\mathcal{Z})$ and define $P_G = G \# P_Z$. We then have that
\begin{align*}
 \max_{\varphi \in \Phi} \inf_{q \in \mathscr{F}\bracket{\mathcal{X}, \mathscr{P}(\mathcal{X})}} \mathcal{L}_{P_X,P_G}(\varphi, q) = \operatorname{WAE}_{c',\text{IPM}_{\Phi \circ G} },
\end{align*}
where $c'(x,y) = - x \cdot y$. Furthermore, we have $(\nabla \varphi^{*} \circ G) \# P_Z = P_X$ and $q^{*} = G \circ E^{*}$ where $E^{*}$ is the optimal encoder from $\operatorname{WAE}_{c',\text{IPM}_{\Phi \circ G} }$.
\end{theorem}
(Proof in Section \ref{proof:brenier-main}).

\subsection{Proof of Theorem \ref{pf-pb-learning}}
\label{proof:pf-pb-learning}
We first begin by citing a well-known result \citep{lei2018geometric} (Theorem 3.7) based on the $c = \nrm{x - y}^2 / 2$ cost function.
\begin{theorem}
\label{gradientispushforward}
Let $\mathcal{X} \subset \mathbb{R}^n$ be a compact domain and $c(x,y) = \nrm{x-y}^2 /2$, then for any $P,Q \in \mathcal{P}(\mathcal{X})$ where at least one of them is absolutely continuous with respect to the lebesgue measure, let $(\varphi_{P,Q}^{*}, \psi_{P,Q}^{*})$ denote the maximizers of \begin{align}\label{briener-1} \sup_{\varphi, \psi \in \mathcal{C}(\mathcal{X}) : \varphi(x) + \psi(y) \leq c(x,y)} \braces{\int_{\mathcal{X}} \varphi(x) dP(x) +  \int_{\mathcal{X}} \psi(y) dQ(y) }. \end{align} Then the map $T^{*}(x) = x - \nabla\varphi_{P,Q}^{*}$ satisfies $T^{*}\# P = Q$ and is the unique minimizer of \begin{align} \label{energy-transport-opt} \inf_{T : T_{\#}P = Q} \int_{\mathcal{X}} c(x,T(x)) dP(x).\end{align}
\end{theorem}
The optimization problem in Equation (\ref{briener-1}) can be reparametrized by setting $\varphi(x) \gets \nrm{x}^2 / 2 - \varphi(x)$ and $\psi(y) \gets \nrm{y}^2 / 2 - \psi(y)$, and so the constraint changes:
\begin{align}
\varphi(x) + \psi(y) \leq c(x,y) &\implies \frac{1}{2}\bracket{\nrm{x}^2 +\nrm{y}^2 } - \bracket{\varphi(x) + \psi(y)} \leq \nrm{x - y}^2 / 2\\
				        &\implies  \frac{1}{2}\bracket{\nrm{x}^2 +\nrm{y}^2 } - \bracket{\varphi(x) + \psi(y)} \leq \frac{1}{2}\bracket{\nrm{x}^2 +\nrm{y}^2 } - \ip{x}{y}\\
				        &\implies  - \bracket{\varphi(x) + \psi(y)} \leq  - \ip{x}{y}\\
				        &\implies \psi(y) \geq \ip{x}{y} - \varphi(x)
\end{align}
The optimal map now becomes $T^{*}(x) = x - \nabla\bracket{\nrm{x}^2 / 2 - \varphi_{P,Q}^{*}} = \nabla\varphi_{P,Q}^{*}$. Finally, we may write the reparametrized optimization problem as
\begin{align}
 &\sup_{\varphi, \psi \in \mathcal{C}(\mathcal{X}) : \varphi(x) + \psi(y) \leq c(x,y)} \braces{\int_{\mathcal{X}} \varphi(x) dP(x) +  \int_{\mathcal{X}} \psi(y) dQ(y) } \\
&= \sup_{\varphi, \psi \in \mathcal{C}(\mathcal{X}) :  \psi(y) \geq \ip{x}{y} - \varphi(x)} \braces{\int_{\mathcal{X}} \bracket{\frac{\nrm{x}^2}{2} - \varphi(x)} dP(x) + \int_{\mathcal{X}} \bracket{ \frac{\nrm{x}^2}{2} - \psi(y)} dQ(y) } \\
&=  \frac{1}{2} \bracket{\int_{\mathcal{X}}\nrm{x}^2 dP(x) + \int_{\mathcal{X}} \nrm{y}^2 dQ(y) }\\ &- \sup_{\varphi, \psi \in \mathcal{C}(\mathcal{X}) :  \psi(y) \geq \ip{x}{y} - \varphi(x)} \braces{\int_{\mathcal{X}} \varphi(x) dP(x)+ \int_{\mathcal{X}} \psi(y) dQ(y) },
\end{align}
and noting that we can set $\psi(y) := \varphi^{\star}(y) = \sup_{x} \braces{ \ip{x}{y} - \varphi(x)}$ since it leads to extreme points (where the solution exists). Indeed, the optimal solution pair $(\varphi_{P,Q}^{*}, \psi_{P,Q}^{*})$ satisfy $\varphi_{P,Q}^{*}(x) + \psi_{P,Q}^{*}(y) = \ip{x}{y}$ which implies that $\psi_{P,Q}^{*} = (\varphi_{P,Q}^{*})^{\star}$ and so $\varphi_{P,Q}^{*}$ is convex. Taking this into consideration, using $\Phi$ to be the set of all convex functions as defined in the statement of Theorem \ref{pf-pb-learning}, the final optimization problem is thus
\begin{align}
\inf_{\varphi \in \Phi} \braces{\int_{\mathcal{X}} \varphi(x) dP(x) + \int_{\mathcal{X}} \varphi^{\star} dQ(x)},
\end{align}
with the property that the optimal $\varphi_{P,Q}^{*}$ is such that $\nabla \varphi_{P,Q}^{*} \# P = Q$ and is the optimizer of Equation (\ref{energy-transport-opt}). We note that the proof up until this point was inspired by the proof for Theorem (\ref{gradientispushforward}). For the remainder, we first require a lemma.
\begin{lemma}
\label{concave-swap-exp}
Let $f: \mathcal{X} \times \mathcal{X} \to \mathbb{R}$ be a function that is concave in the first argument, then we have \begin{align} \label{eq:probdeter}
    \sup_{M \in \mathscr{F}(\mathcal{X}, \mathbb{R})}\E_{y \sim Q}[f(M(y),y)] = \sup_{q \in \mathcal{F}(\mathcal{X}, \mathscr{P}(\mathcal{X}))}\E_{y \sim Q}[\E_{x \sim q(y)}[f(x,y)]].
\end{align}
\end{lemma}
\begin{proof}
For each $M \in \mathscr{F}(\mathcal{X},\mathbb{R})$, we can associate $\delta_{M(y)} \in \mathscr{P}(\mathcal{X})$ since $\E_{y \sim Q}[f(M(y),y)] = \E_{y \sim Q}[\E_{x \sim \delta_{M(y)}}[f(x,y)]]$. Hence we have that \begin{align*}\sup_{M \in \mathscr{F}(\mathcal{X}, \mathbb{R})}\E_{y \sim Q}[f(M(y),y)]\leq \sup_{q \in \mathcal{F}(\mathcal{X}, \mathscr{P}(\mathcal{X}))}\E_{y \sim Q}[\E_{x \sim q(y)}[f(x,y)]], \end{align*} since the right hand side searches over a larger space. For any $q$, denote $M_q(y) := \E_{x \sim q(y)}[x] \in \mathscr{F}(\mathcal{X}, \mathbb{R})$, by Jensen's inequality we have
\begin{align*}
\E_{y \sim Q}[\E_{x \sim q(y)}[f(x,y)] &\leq \E_{y \sim Q}[f\bracket{\E_{x \sim q(y)}[x],y}]\\
						&= \E_{y \sim Q}[f(M_q(y),y)]\\
						&\leq \sup_{M \in \mathscr{F}(\mathcal{X},\mathbb{R})} \E_{y \sim Q}[f(M(y), y)]
\end{align*}
Hence we have equality. We note that the explicit correspondence to go from $\mathscr{F}(\mathcal{X}, \mathbb{R})$ to $\mathscr{P}(\mathcal{X})$ is $M \to \delta_{M}$ and from $\mathscr{P}(\mathcal{X})$ to $\mathscr{F}(\mathcal{X}, \mathbb{R})$ is $q \to M_q$.
\end{proof}
In Lemma \ref{concave-swap-exp}, we set $f(x,y) = \ip{x}{y} - \varphi(x)$, which is concave in the first argument if $\varphi(x)$, to get
\begin{align*}
\sup_{M \in \mathscr{F}(\mathcal{X}, \mathbb{R})}\E_{y \sim Q}[ \ip{M(y)}{y} - \varphi(M(y))] = \sup_{q \in \mathcal{F}(\mathcal{X}, \mathscr{P}(\mathcal{X}))}\E_{y \sim Q}[\E_{x \sim q(y)}[ \ip{x}{y} - \varphi(x)]
\end{align*}
The optimization then becomes
\begin{align*}&\inf_{\varphi \in \Phi} \braces{\int_{\mathcal{X}}\varphi(x) dP(x) + \int_{\mathcal{X}} \varphi^{\star}(y) dQ(y) }\\
    &= \inf_{\varphi \in \Phi}\braces{\int_{\mathcal{X}}\varphi(x) dP(x) + \int_{\mathcal{X}} \sup_{x} \bracket{x \cdot y - \varphi(x) } dQ(y) }\\
    &= \inf_{\varphi \in \Phi}\braces{\int_{\mathcal{X}}\varphi(x) dP(x) + \sup_{M \in \mathscr{F}(\mathcal{X}, \mathbb{R})}\int_{\mathcal{X}} M(y) \cdot y - \varphi(M(y))  dQ(y) }\\
&= \inf_{\varphi \in \Phi}\braces{\int_{\mathcal{X}}\varphi(x) dP(x) + \sup_{q \in \mathcal{F}(\mathcal{X}, \mathscr{P}(\mathcal{X}))}\int_{\mathcal{X}} \braces{\E_{x \sim q(y)}[x \cdot y] - \E_{x \sim q(y)}[\varphi(x)]}  dQ(y) }\\
&= \inf_{\varphi \in \Phi} \sup_{q \in \mathcal{F}(\mathcal{X}, \mathscr{P}(\mathcal{X}))}\braces{\int_{\mathcal{X}}\varphi(x) dP(x) +\int_{\mathcal{X}} \braces{\E_{x \sim q(y)}[x \cdot y] - \E_{x \sim q(y)}[\varphi(x)]}  dQ(y) }\\
&=\max_{\varphi \in \Phi} \inf_{q \in \mathcal{F}(\mathcal{X}, \mathscr{P}(\mathcal{X}))} \braces{\int_{\mathcal{X}} \E_{x \sim q(y)} [\varphi(x)] dQ(y)  - \int_{\mathcal{X}} \varphi(x) dP(x) - \int_{\mathcal{X}} \E_{x \sim q(y)} [x \cdot y] dQ(y)},
\end{align*}
arriving at the expression for $\mathcal{L}_{P,Q}$. Hence the maximizer of $\varphi^{*}$ of
\begin{align}
\max_{\varphi \in \Phi} \inf_{q \in \mathscr{F}\bracket{\mathcal{X}, \mathscr{P}(\mathcal{X})}} \mathcal{L}_{P,Q}(\varphi, q)
\end{align}
satisfies $\nabla \varphi^{*} \# P = Q$. To prove the moment matching property of $q^{*}$, we require a lemma first.
\begin{lemma}
For a fixed $\varphi$, if $q_{\varphi}$ is the minimizer of $\inf_{q \in \mathcal{F}(\mathcal{X}, \mathscr{P}(\mathcal{X}))}\mathcal{L}_{P,Q}(\varphi,q)$, then we have
\begin{align*}
\nabla \varphi \bracket{\E_{x \sim q_{\varphi}(y)}[x]} = y.
\end{align*}
\end{lemma}
\begin{proof}
By rewritng
\begin{align}
&\inf_{q \in \mathcal{F}(\mathcal{X}, \mathscr{P}(\mathcal{X}))} \mathcal{L}_{P,Q}(\varphi,q)\\ &= \inf_{q \in \mathcal{F}(\mathcal{X}, \mathscr{P}(\mathcal{X}))}\braces{\int_{\mathcal{X}} \E_{x \sim q(y)} [\varphi(x)] dQ(y)  - \int_{\mathcal{X}} \varphi(x) dP(x) - \int_{\mathcal{X}} \E_{x \sim q(y)} [x \cdot y] dQ(y)} \nonumber \\
&= \int_{\mathcal{X}}\varphi(x) dP(x) + \sup_{q \in \mathcal{F}(\mathcal{X}, \mathscr{P}(\mathcal{X}))}\int_{\mathcal{X}} \braces{\E_{x \sim q(y)}[x \cdot y] - \E_{x \sim q(y)}[\varphi(x)]}  dQ(y)\label{opt-for-q}  \\
&= \int_{\mathcal{X}} \varphi(x) dP(x) +  \sup_{M \in \mathscr{F}(\mathcal{X}, \mathbb{R})}\int_{\mathcal{X}} M(y) \cdot y - \varphi(M(y))  dQ(y) \label{opt-for-m}
\end{align}
In Equation (\ref{opt-for-m}), $x = M(y)$ is such that the expression $x \cdot y - \varphi(x)$ is maximized over $x$ for a particular $y$ (and hence the dependence of $y$). Due to concavity of this equation (through convexity of $\varphi$), one can obtain the solution by solving $\nabla_x \bracket{x \cdot y - \varphi(x)} = 0$, which leads to the equation $y -  \nabla \varphi(x) = 0$ and hence the solution would be an $x$ such that $\nabla \varphi(x) = y$, and since we defined it to be $M(y)$, we  have $\nabla \varphi(M(y)) = y$. In the proof of Lemma \ref{concave-swap-exp}, it is shown that there is an explicit correspondence when transferring between Equation (\ref{opt-for-q}) and Equation (\ref{opt-for-m}), which is precisely that $M_q(y) = \E_{x \sim q}[x]$. Thus the optimal $q = q_{\varphi}$ induces a function $M_{q_{\varphi}}$ with the property $\nabla \varphi \bracket{M(y)} = y$ and by using the explicit correspondence, we have $\nabla \varphi \bracket{\E_{x \sim q_{\varphi}(y)  }[x]} = y$, as required.
\end{proof}

Consider now some $g \in \bracket{\mathscr{F}(\mathcal{X}, \mathbb{R}) \circ \nabla \varphi }$, then there exists some $f_g \in \mathscr{F}(\mathcal{X}, \mathbb{R})$ such that $f = f_g \circ \nabla \varphi$. We then have
\begin{align*}
\E_{x \sim Q}[g(M^{*}(x))] &= \E_{x \sim Q}[ f_g(  \nabla \varphi (M^{*}(x) )  ]\\
				   &= \E_{x \sim Q}[f_g (x) ]\\
				   &= \E_{x \sim \nabla \varphi \# P}[f_g(x)]\\
	  			  &= \E_{x \sim P}[f_g(\nabla \varphi(x))]\\
				  &= \E_{x \sim P}[g(x)],
\end{align*}
Hence, $M^{*} \# Q$ and $P$ are matched under the moments of $\mathscr{F}(\mathcal{X}, \mathbb{R}) \circ \nabla \varphi$
\subsection{Proof of Theorem \ref{brenier-main}}
\label{proof:brenier-main}
Suppose we fix $G$ and re-parametrize $q = G \circ E$. The objective $\mathcal{L}_{P_G,P_X}$ decomposes into
\begin{align*}
\mathcal{L}_{P_G,P_X}(\varphi,q) &= \int_{\mathcal{X}} \E_{x \sim G \circ E(y)} [\varphi(x)] dP_X(y)  - \int_{\mathcal{X}} \varphi(x) dP_G(x) - \int_{\mathcal{X}} \E_{x \sim G \circ E(y)} [x \cdot y] dP_X(y)\\
					   &=  \int_{\mathcal{X}} \varphi(G(z)) dE \#P_X(z)  - \int_{\mathcal{X}} \varphi(x) dP_G(x) - \int_{\mathcal{X}} \E_{x \sim G \circ E(y)} [x \cdot y] dP_X(y)\\
					   &= \int_{\mathcal{X}} \varphi(G(z)) dE \#P_X(z)  - \int_{\mathcal{X}} \varphi(G(z)) dP_Z(z) - \int_{\mathcal{X}} \E_{x \sim G \circ E(y)} [x \cdot y] dP_X(y)
\end{align*}
By strong duality (since $\mathcal{L}_{P_G,P_X}$ is bilinear), we have
\begin{align*}
&\max_{\varphi \in \Phi}\inf_{E \in \mathscr{F}(\mathcal{X}, \mathscr{P}(\mathcal{X})) } \mathcal{L}_{P_G,P_X}\\ &= \inf_{E \in \mathscr{F}(\mathcal{X}, \mathscr{P}(\mathcal{X})) } \max_{\varphi \in \Phi} \mathcal{L}_{P_G,P_X}\\
&=  \inf_{E \in \mathscr{F}(\mathcal{X}, \mathscr{P}(\mathcal{X})) }\braces{ \max_{\varphi \in \Phi}\braces{\int_{\mathcal{X}} \varphi(G(z)) dE \#P_X(z)  - \int_{\mathcal{X}} \varphi(G(z)) dP_Z(z)}    - \int_{\mathcal{X}} \E_{x \sim G \circ E(y)} [x \cdot y] dP_X(y) }\\
&= \inf_{E \in \mathscr{F}(\mathcal{X}, \mathscr{P}(\mathcal{X})) }\braces{ \text{IPM}_{\Phi \circ G}(P_Z, E \# P_X) - \int_{\mathcal{X}} \E_{x \sim G \circ E(y)} [x \cdot y] dP_X(y)  }\\
& = \inf_{E \in \mathscr{F}(\mathcal{X}, \mathscr{P}(\mathcal{X})) }\braces{ \text{IPM}_{\Phi \circ G}(P_Z, E \# P_X) - \int_{\mathcal{X}} \E_{z \sim  E(y)} [G(z) \cdot y] dP_X(y)  }\\
&= \text{WAE}_{c',\text{IPM}_{\Phi \circ G} },
\end{align*}
where $c'(x,y) = -x \cdot y$. From Theorem \ref{pf-pb-learning}, we have that $\nabla \varphi^{*} \# P_G = P_X$ and noting that $P_G = G \# P_Z$ along with associativity of pushfoward, we get $P_X = \nabla \varphi^{*} \# (G \# P_Z) = (\nabla \varphi^{*} \circ G) \# P_Z$. Finally, since we reparametrized $q$ with $q = G \circ E$, it is clear that at the optima, we have the relationship $q^{*} = G \circ E^{*}$.

\end{document}